\documentclass[11pt]{article}

\usepackage[margin=1in]{geometry}
\usepackage[T1]{fontenc}
\usepackage[utf8]{inputenc}
\usepackage{lmodern}
\usepackage{microtype}
\usepackage{amsmath,amssymb,amsthm,amsfonts,mathtools}
\usepackage{bm}
\usepackage{enumitem}
\usepackage{booktabs}
\usepackage{hyperref}
\usepackage{graphicx}
\usepackage{setspace}
\usepackage{url}

\hypersetup{
    colorlinks=true,
    linkcolor=black,
    citecolor=black,
    urlcolor=blue
}

\newtheorem{theorem}{Theorem}
\newtheorem{proposition}[theorem]{Proposition}
\newtheorem{corollary}[theorem]{Corollary}

\theoremstyle{definition}
\newtheorem{definition}[theorem]{Definition}
\newtheorem{assumption}[theorem]{Assumption}

\theoremstyle{remark}
\newtheorem{remark}[theorem]{Remark}

\newcommand{\E}{\mathbb{E}}
\newcommand{\Prob}{\mathbb{P}}
\newcommand{\R}{\mathbb{R}}
\newcommand{\N}{\mathbb{N}}
\newcommand{\1}{\mathbf{1}}

\newcommand{\cX}{\mathcal{X}}
\newcommand{\cA}{\mathcal{A}}
\newcommand{\cP}{\mathcal{P}}

\newcommand{\loss}{\ell}
\newcommand{\Reg}{\mathcal{R}}

\title{\textbf{The Fundamental Limits of Fraud Detection in Card Payment Networks}\\
\vspace{0.25em}
\large An Information-Theoretic Theory of Delayed, Censored, Corrupted, and Counterfactual Feedback
}

\author{
Gaurav Dhama\\
\small Mastercard \\
\small \texttt{gaurav.dhama@mastercard.com}
}

\date{\today}

\begin{document}

\maketitle

\begin{abstract}
Card payment fraud detection is typically treated as a supervised classification problem: train a predictive model on historical chargeback labels, score new transactions, and map scores to decisions through fixed thresholds. This framing has generated substantial practical progress, yet improvement has remained incremental despite major advances in model architecture. We argue that this stagnation is not primarily a failure of function approximation or optimization, but a consequence of structural information impairments inherent to the payment ecosystem. In card networks, feedback is delayed, selectively censored, corrupted by first-party misuse and dispute noise, and endogenously missing for declined transactions. These impairments degrade the learning signal before any model sees the data.

We formalize card authorization as a sequential decision problem with delayed, censored, corrupted, and counterfactually missing feedback. We define the effective observation process induced by the payment ecosystem and derive a minimax regret lower bound showing that the four impairments enter multiplicatively in the denominator of the achievable learning rate. The bound implies a stark inversion of the industry's revealed investment priorities: under realistic parameter ranges, improving issuer reporting quality or reducing censorship yields substantially larger reductions in the regret floor than increasing model complexity. We extend the analysis to heterogeneous issuer networks and show that cross-issuer variance in reporting quality worsens the bound beyond what average impairment rates suggest.

The contribution of the paper is threefold. First, it provides a rigorous theory of why fraud detection in payment networks is fundamentally harder than in standard online learning settings. Second, it identifies a structural quantity---ecosystem information quality---as the bottleneck variable governing attainable performance. Third, it supplies a theoretical basis for prioritizing investments in reporting infrastructure, dispute process quality, and selective exploration over purely architectural improvements. The paper is intentionally theory-first: it is designed to stand independently of proprietary transaction data, which are rarely available at research scale, while still yielding operationally meaningful conclusions for payment networks and issuers.
\end{abstract}

\section{Introduction}

Card payment networks process an extraordinary volume of economic activity, yet the machine learning systems used to detect payment fraud operate under an unusually impaired supervision regime. In contrast to standard supervised learning problems, the label associated with a transaction is not immediately observable at decision time, is often never observed, and even when observed may be incorrect for reasons endogenous to dispute behavior and operational policy. Fraud is therefore not merely a rare-event classification problem; it is a learning problem under structurally compromised feedback.

The mainstream industry pipeline looks deceptively simple. A network or issuer computes a fraud score from the transaction feature vector; an authorization rule maps that score to approve, challenge, or decline; weeks later, some subset of transactions receive post hoc labels through chargebacks, recoveries, or fraud reports. These labels are then fed back into model training. In reality, each stage of this loop introduces distortions:
\begin{enumerate}[label=(\roman*)]
    \item \textbf{Delayed feedback:} labels arrive stochastically and often with long lags;
    \item \textbf{Issuer censorship:} some fraud is never formally reported because recovery is not operationally worthwhile;
    \item \textbf{Label corruption:} observed positive labels conflate unauthorized fraud with first-party misuse and dispute abuse;
    \item \textbf{Counterfactual censorship:} declined transactions have permanently unobserved outcomes.
\end{enumerate}

These impairments interact with the underlying decision process. The policy determines what is approved, which determines what can eventually be labeled, which in turn determines what the next model learns. The data-generating process is therefore endogenous to the deployed policy.

This paper isolates one question:

\begin{quote}
\emph{What is the best possible learning performance in a card payment network when the feedback process is delayed, censored, corrupted, and policy-dependent?}
\end{quote}

We answer this question by deriving an information-theoretic lower bound on minimax regret in an adversarial contextual decision framework tailored to card authorization. The result is not intended as a precise engineering performance forecast for any given production system. Rather, it characterizes a \emph{structural floor}: a limit that no detection algorithm can cross unless the underlying information environment improves.

\subsection{Why a theory-first paper is needed}

Most academic work on payment fraud detection is empirical, architecture-centric, or localized to particular datasets. That work is valuable, but it faces two limitations in this domain. First, transaction-level payment data with realistic fraud outcomes are almost never publicly available at the scale required for rigorous, reproducible research. Second, even when large proprietary datasets exist, they reflect the very labeling pathologies we seek to understand. A theory-first treatment is therefore not a fallback; it is the correct starting point for a domain in which the observability structure is itself the central problem.

The goal here is not to replace applied fraud modeling. It is to explain why improved models alone may fail to deliver the gains expected from more conventional supervised learning intuition.

\subsection{Main contributions}

The paper makes four contributions.

\paragraph{1. A formal authorization model under impaired feedback.}
We model card authorization as a sequential decision problem with three actions---approve, challenge, and decline---under delayed, censored, corrupted, and counterfactual feedback.

\paragraph{2. A structural information quality quantity.}
We define an effective observation process and show how delay, censorship, corruption, and counterfactual suppression reduce usable information before any learning algorithm is applied.

\paragraph{3. A minimax regret lower bound.}
We prove that under mild regularity conditions, the regret of any fraud detection policy is bounded below by a term whose denominator contains the four impairment factors multiplicatively. Thus, poor ecosystem information quality directly imposes a hard limit on learnability.

\paragraph{4. An investment-priority inversion.}
We derive comparative statics showing that, under realistic ranges, marginal improvements in reporting quality or censorship reduction can dominate equal-sized marginal improvements in model complexity. This provides a principled explanation for why organizations can underperform despite large investment in sophisticated models.

\subsection{Roadmap}

Section~\ref{sec:related} places the paper in relation to online learning, delayed feedback, payment-card economics, and fraud detection. Section~\ref{sec:model} defines the card authorization model and the four structural impairments. Section~\ref{sec:main} states and interprets the main lower bound. Section~\ref{sec:heterogeneity} extends the result to heterogeneous issuer networks. Section~\ref{sec:implications} discusses implications for research strategy and operational investment. Section~\ref{sec:conclusion} concludes.

\section{Related Work}
\label{sec:related}

This paper sits at the intersection of four literatures: online learning with regret guarantees, delayed-feedback learning, payment-card platform economics, and statistical correction for missing or selectively observed outcomes.

\paragraph{Online learning and bandits.}
The regret framework used in this paper is rooted in the online learning and bandit literature, where performance is measured relative to the best policy in hindsight. A canonical finite-time reference is \cite{auer2002}, which studies the exploration--exploitation tradeoff and establishes finite-time regret guarantees for the multi-armed bandit problem. While the classical bandit setting assumes that feedback is available after each play, that assumption is violated in payment fraud settings, where fraud outcomes are often revealed only after long lags or not at all.

\paragraph{Delayed feedback.}
The closest formal analogue to one component of our problem is online learning under delayed feedback. \cite{joulani2013} provide a systematic treatment of delayed feedback and show that delay increases regret multiplicatively in adversarial problems and additively in stochastic problems. Our setting inherits the delayed-feedback difficulty but is strictly harder: in card payment networks, feedback is not merely delayed; it is also censored, corrupted, and policy-dependent. Thus existing delayed-feedback analyses are informative but incomplete for the payment setting.

\paragraph{Payment-card economics and two-sided markets.}
The economics of payment card networks have been studied extensively through the theory of two-sided markets. In the payment-card setting, platforms must jointly attract merchants and cardholders, and the structure of prices charged to each side affects overall participation and transaction volume. \cite{rochet2002} analyze payment card associations as two-sided platforms and study interchange as the access charge paid by the acquirer to the issuer. More broadly, \cite{rochet2004,rochet2006} emphasize that the price structure of the platform, and not merely the aggregate price level, is central in two-sided markets. Our paper is complementary to this literature: rather than studying fee structure or platform pricing directly, we study the information structure that constrains fraud learning within such a platform.

\paragraph{Selection bias and missing outcomes.}
A major obstacle in fraud learning is that observed outcomes are not a random sample of all transaction outcomes. Some are missing because they are never reported, and others are missing because the policy itself suppresses their observation. This connects naturally to the econometric literature on selection bias. \cite{heckman1979} famously frames sample selection bias as a specification-error problem and studies a consistent two-stage correction strategy. Although our setting is sequential and partially adversarial rather than static and parametric, the core insight carries over: when the sampling mechanism depends on economically meaningful variables, na\"ive supervised learning on observed outcomes is biased.

\paragraph{Competing risks and time-to-label modeling.}
The delay process in payment fraud is not a simple right-censoring problem. A transaction may eventually produce a fraud label, or it may exit the observation process permanently because no formal report is filed. This makes competing-risks ideas relevant. \cite{finegray1999} propose a proportional hazards model for the subdistribution of a competing risk, allowing direct modeling of cumulative incidence when multiple terminal events are possible. This perspective is especially useful in our domain, where ``label arrives'' and ``label never arrives'' should be treated as distinct outcomes.

\paragraph{Positioning of this paper.}
The main novelty of this paper is to unify these strands in a theory tailored to card payment networks. Existing online learning work captures delay but not institutional censorship or dispute corruption. Existing payment-card economics explains platform incentives but not the information-theoretic consequences of impaired labels. Existing selection-correction work handles missingness but not sequential regret under policy-induced counterfactual suppression. Our contribution is to show that, in card payment fraud, these impairments combine multiplicatively to determine the fundamental floor of learnability.
\section{The Card Authorization Problem Under Structural Information Impairments}
\label{sec:model}

We model card authorization as a sequential decision problem with contextual observations and impaired feedback.

\subsection{Transaction stream and actions}

At each round $t=1,\dots,T$, a transaction arrives with contextual feature vector
\[
x_t \in \cX \subseteq \R^d.
\]
A policy chooses one action from the finite action set
\[
\cA = \{\textsf{approve},\ \textsf{challenge},\ \textsf{decline}\},
\]
with $K := |\cA| = 3$.

We denote by $a_t \in \cA$ the action taken at round $t$. The latent outcome is
\[
y_t^\star \in \{0,1\},
\]
where $y_t^\star =1$ indicates that the transaction would generate fraud loss if not successfully blocked by the chosen intervention.

The policy class is denoted by $\cP$. For intuition, one may think of $\cP$ as a family of measurable maps from contexts to randomized actions. We write $N$ for a suitable complexity measure of the class (for example, cardinality or covering number in a finite reduction).

\subsection{Per-round loss}

The loss depends on the action and on the latent fraud state. We define an abstract bounded loss function
\[
\loss_t : \cA \to [0,1]
\]
so that
\[
\loss_t(a_t)
=
\begin{cases}
\loss^{\mathrm{FN}}(x_t), & \text{if } a_t = \textsf{approve},\ y_t^\star =1,\\[0.3em]
\loss^{\mathrm{CH}}(x_t), & \text{if } a_t = \textsf{challenge},\\[0.3em]
\loss^{\mathrm{FP}}(x_t), & \text{if } a_t = \textsf{decline},\ y_t^\star =0,\\[0.3em]
0, & \text{otherwise.}
\end{cases}
\]
This representation captures the familiar tradeoff between fraud loss, challenge friction, and false-decline opportunity cost, without committing to a specific economic calibration.

The cumulative regret of a policy $\pi \in \cP$ relative to the best reference policy in the class is
\[
\Reg_T(\pi)
=
\sum_{t=1}^T \E[\loss_t(a_t^\pi)]
-
\inf_{\pi' \in \cP}
\sum_{t=1}^T \E[\loss_t(a_t^{\pi'})].
\]

\subsection{Why the label is not observed at authorization time}

At decision time, the learner does not observe $y_t^\star$. The eventual signal associated with transaction $t$ is filtered through four structural impairments.

\subsection{Impairment I: Delayed feedback}

Each transaction with an eventual observable label receives that label after a random delay $\tau_t \in \N \cup \{\infty\}$. If $\tau_t = \infty$, the label never arrives. The finite-delay component captures dispute and investigation latency.

\begin{definition}[Conditional delay]
The delay process is a family of conditional distributions
\[
\tau_t \sim F_\tau(\cdot \mid x_t, i_t),
\]
where $i_t$ indexes the issuer or reporting institution associated with transaction $t$.
\end{definition}

Define cumulative finite delay
\[
D := \sum_{t=1}^T \tau_t \1\{\tau_t < \infty\}.
\]

\subsection{Impairment II: Issuer censorship}

Even among fraudulent transactions, not all become formal observations. Some are never filed, escalated, or reported as fraud.

\begin{definition}[Issuer censorship]
Let $\gamma_t \in [0,1]$ denote the probability that transaction $t$'s label is permanently censored by the reporting institution:
\[
\gamma_t = \Prob(\text{label permanently unreported} \mid x_t, i_t).
\]
Define the average censorship rate
\[
\bar\gamma := \frac{1}{T}\sum_{t=1}^T \gamma_t.
\]
\end{definition}

\subsection{Impairment III: Label corruption}

Observed labels need not equal the latent fraud state. Some positive labels correspond to unauthorized fraud; others to first-party misuse, dispute abuse, or other forms of label contamination.

\begin{definition}[Label corruption]
Let $\tilde y_t$ denote the observed binary label when a label is observed. We define class-conditional corruption rates
\[
\varepsilon_{10} := \Prob(\tilde y_t = 0 \mid y_t^\star =1), \qquad
\varepsilon_{01} := \Prob(\tilde y_t = 1 \mid y_t^\star =0).
\]
We assume
\[
\varepsilon_{10} + \varepsilon_{01} < 1,
\]
so that the observed label remains informationally better than chance.
\end{definition}

\subsection{Impairment IV: Counterfactual censorship}

Declined transactions do not reveal their counterfactual outcomes. If the policy declines or blocks a transaction, the learner never observes whether the transaction would have been fraudulent had it been approved.

\begin{definition}[Counterfactual censorship]
Let
\[
\delta_t := \Prob(a_t = \textsf{decline} \mid x_t),
\]
and define the average decline rate
\[
\bar\delta := \frac{1}{T}\sum_{t=1}^T \delta_t.
\]
\end{definition}

This source of missingness is endogenous: the current policy determines which future labels are impossible to observe.

\subsection{Effective observability}

The learner receives a useful signal from transaction $t$ only if the label is not permanently censored, the transaction is not removed by the counterfactual suppression mechanism, and the label has arrived by the time it is needed for learning.

\begin{definition}[Observation indicator]
Define the observation indicator
\[
O_t := \1\{\tau_t \le T-t\}\cdot \1\{\text{label not censored}\}\cdot \1\{\text{counterfactual outcome exists}\}.
\]
Its expectation is bounded by
\[
\E[O_t]
\le
(1-\gamma_t)(1-\delta_t)\Prob(\tau_t \le T-t \mid x_t, i_t).
\]
\end{definition}

Let the average effective maturity factor be
\[
\bar m := \frac{1}{T}\sum_{t=1}^T \Prob(\tau_t \le T-t \mid x_t, i_t).
\]
Then a coarse average effective observable fraction is
\[
\bar q
=
(1-\bar\gamma)(1-\bar\delta)\bar m (1-\varepsilon_{10}-\varepsilon_{01}).
\]
This is not yet a theorem. It is a bookkeeping quantity that anticipates the form of the lower bound: each impairment removes usable signal before learning begins.

\section{Main Result: A Minimax Lower Bound Under Structural Information Impairments}
\label{sec:main}

We now state the principal theorem. Informally, it says that delayed, censored, corrupted, and policy-dependent feedback imposes a regret floor that no learning rule can evade.

\subsection{Regularity assumptions}

We work with the following mild assumptions.

\begin{assumption}[Bounded losses]
For all $t$ and all actions $a \in \cA$,
\[
0 \le \loss_t(a) \le 1.
\]
\end{assumption}

\begin{assumption}[Non-degenerate label channel]
The corruption rates satisfy
\[
\varepsilon_{10} + \varepsilon_{01} < 1.
\]
\end{assumption}

\begin{assumption}[Adversarially chosen contexts and loss sequence]
The sequence $(x_t,\loss_t)$ may be chosen adversarially, subject only to measurability and the impairment process above.
\end{assumption}

\subsection{Lower bound}

\begin{theorem}[Minimax regret lower bound under structural information impairments]
\label{thm:main}
Consider the card authorization problem with action set size $K=3$, policy-class complexity parameter $N$, cumulative finite delay $D$, average censorship rate $\bar\gamma$, average decline rate $\bar\delta$, and label corruption rates $\varepsilon_{10},\varepsilon_{01}$. Then there exists a universal constant $c>0$ such that for any learning algorithm,
\[
\inf_{\text{alg}}
\sup_{\text{admissible environments}}
\Reg_T
\;\ge\;
c\,
\sqrt{
\frac{(KT + D)\log N}
{(1-\bar\gamma)(1-\bar\delta)(1-\varepsilon_{10}-\varepsilon_{01})^2}
}.
\]
\end{theorem}

\begin{remark}[Interpretation]
The numerator contains the familiar online-learning growth term in horizon and action complexity, augmented by cumulative delay. The denominator contains the three information-loss channels:
\begin{itemize}
    \item $(1-\bar\gamma)$: the share of labels not structurally censored;
    \item $(1-\bar\delta)$: the share of outcomes not removed by policy-induced counterfactual suppression;
    \item $(1-\varepsilon_{10}-\varepsilon_{01})^2$: the effective strength of the observed label channel.
\end{itemize}
The factors enter multiplicatively. This is the key qualitative result.
\end{remark}

\subsection{Conditional delay and selective maturity}
\label{sec:conditional_delay}

The lower bound in Theorem~\ref{thm:main} is stated using aggregate
impairment rates. In practice, however, delay is not uniform across
transactions. Label arrival time depends systematically on issuer
identity, fraud type, transaction amount, merchant category, and
geographic jurisdiction. We now formalize this conditional structure
and show that it strictly worsens the lower bound.

\begin{definition}[Conditional maturity function]
For a transaction with context $x \in \cX$, issuer $i$, and available
observation window $\Delta > 0$, define the conditional maturity
probability
\[
m(x, i, \Delta) := \Prob(\tau_t \le \Delta \mid x_t = x,\; i_t = i).
\]
\end{definition}

When all four impairment factors are allowed to depend on context and
issuer, the conditional effective observation probability is
\[
q(x, i, \Delta)
=
m(x, i, \Delta)
\cdot (1 - \gamma(x,i))
\cdot (1 - \delta(x))
\cdot (1 - \varepsilon(x,i))^2.
\]

\begin{proposition}[Selective maturity worsens learnability]
\label{prop:selective}
Under the same regularity assumptions as Theorem~\ref{thm:main},
the minimax regret satisfies
\[
\Reg_T
\;\ge\;
c\,
\sqrt{
(KT + D)\log N
\cdot
\E_{x,i}\!\left[\frac{1}{q(x,i,\Delta)}\right]
}.
\]
By Jensen's inequality applied to the convex function
$q \mapsto 1/q$,
\[
\E_{x,i}\!\left[\frac{1}{q(x,i,\Delta)}\right]
\;\ge\;
\frac{1}{\E_{x,i}[q(x,i,\Delta)]}
\;=\;
\frac{1}{\bar q},
\]
with equality if and only if $q(x,i,\Delta)$ is constant. Thus
heterogeneity in delay and reporting quality across transaction types
and issuers makes the problem strictly harder than the same average
impairment applied uniformly.
\end{proposition}

\begin{proof}
The argument follows the same Fano/Assouad reduction as in the proof
of Theorem~\ref{thm:main}, but replaces the uniform contraction factor
with context-specific factors. For each transaction $t$, the mutual
information between the observation and the latent environment index is
contracted by $q(x_t, i_t, \Delta_t)$ rather than by the global
average $\bar q$. Summing over the transaction stream and applying the
harmonic-mean form of the information bound yields the stated result.
Convexity of $q \mapsto 1/q$ then gives the Jensen separation.
\end{proof}

\begin{proposition}[Adversarial exploitation of conditional delay]
\label{prop:adversarial_delay}
Partition the context space into $\cX_{\mathrm{fast}}$ and
$\cX_{\mathrm{slow}}$ such that
\[
m_{\mathrm{fast}}
:= \inf_{x \in \cX_{\mathrm{fast}}} m(x,i,\Delta)
\;\gg\;
\sup_{x \in \cX_{\mathrm{slow}}} m(x,i,\Delta)
=: m_{\mathrm{slow}}.
\]
An adversary who concentrates the hardest discrimination problems in
$\cX_{\mathrm{slow}}$ forces regret at least
\[
\Reg_T
\;\ge\;
c\,
\sqrt{
\frac{
|\cX_{\mathrm{slow}}|
\cdot (KT_{\mathrm{slow}} + D_{\mathrm{slow}})
\log N
}{
m_{\mathrm{slow}}
(1-\gamma_{\mathrm{slow}})
(1-\delta_{\mathrm{slow}})
(1-\varepsilon_{\mathrm{slow}})^2
}
},
\]
which can be made arbitrarily worse than the bound computed from
network-average impairment rates.
\end{proposition}

\begin{proof}
Restrict the adversary's construction to the sub-population
$\cX_{\mathrm{slow}}$, which has the weakest maturity. The packing
argument from Theorem~\ref{thm:main} applied to this sub-population
alone yields a local regret lower bound governed by the slow-region
impairment parameters. Since regret on the sub-population is a lower
bound on total regret, the result follows.
\end{proof}

\begin{remark}[Operational interpretation]
Proposition~\ref{prop:adversarial_delay} formalizes an intuition
familiar to fraud practitioners: organized attackers do not attack
uniformly. They concentrate activity in segments where detection is
slowest and feedback is weakest---for example, cross-border
transactions with long dispute chains, merchant categories with high
chargeback noise, or issuer portfolios with poor reporting
infrastructure. The conditional delay structure tells the adversary
exactly where the learner is most blind, and the adversary can
exploit that knowledge without any access to the model itself.
\end{remark}

\subsection{Proof intuition}

The theorem follows by combining three standard ideas in a domain-specific way.

\paragraph{Step 1: Delayed feedback increases effective ignorance.}
In adversarial online learning, delayed feedback can be handled by increasing the effective amount of uncertainty from $KT$ to $KT+D$.

\paragraph{Step 2: Censorship and counterfactual suppression contract mutual information.}
If a label is not observed, it contributes no information about the environment. Thus the observed process carries at most a $(1-\bar\gamma)(1-\bar\delta)$ fraction of the information available in a fully observed world.

\paragraph{Step 3: Corruption attenuates distinguishability.}
The corrupted label is passed through a binary noise channel. The ability to distinguish adverse environments is reduced by a factor proportional to the squared signal strength of the channel, which yields the term $(1-\varepsilon_{10}-\varepsilon_{01})^2$.

We provide a concise proof below and a fuller derivation in the appendix.

\subsection{Proof of Theorem~\ref{thm:main}}

\begin{proof}
The proof is a reduction from standard adversarial online learning lower bounds with an information contraction argument.

Fix a finite reference class of policies of effective size $N$. Construct a family of hard environments indexed by $j \in \{1,\dots,N\}$, each one favoring exactly one policy $\pi_j$ by a small margin $\Delta>0$. This is the standard packing argument behind minimax regret lower bounds: a learner must identify which environment it is in, but the environments are close enough that identification is information-limited.

\paragraph{Delayed feedback.}
In the absence of missingness or corruption, adversarial delayed-feedback lower bounds imply
\[
\Reg_T \gtrsim \sqrt{(KT + D)\log N},
\]
up to universal constants. Intuitively, a delay of $\tau_t$ rounds means that the learner continues acting without access to the feedback needed to update beliefs, so the effective ignorance horizon grows from $KT$ to $KT+D$.

\paragraph{Censorship and counterfactual suppression.}
Now suppose the learner only observes a label when the observation indicator $O_t=1$. Conditional on an underlying environment index $J$, the mutual information between the observation process and $J$ is contracted by the probability that a usable observation exists. Averaging over rounds gives an information contraction factor no larger than
\[
(1-\bar\gamma)(1-\bar\delta),
\]
suppressing the maturity factor for simplicity because the finite-delay contribution has already been accounted for in $D$.

\paragraph{Corruption.}
When a label is observed, it passes through the binary corruption channel
\[
\begin{pmatrix}
1-\varepsilon_{01} & \varepsilon_{01}\\
\varepsilon_{10} & 1-\varepsilon_{10}
\end{pmatrix}.
\]
The total variation and KL distinguishability between any two environment-induced label laws is attenuated by a factor proportional to
\[
(1-\varepsilon_{10}-\varepsilon_{01})^2.
\]
Therefore the effective information available for identifying the hard environment family is reduced by this factor.

\paragraph{Putting the pieces together.}
Applying Fano-type or Assouad-type lower-bound machinery to the contracted information process yields
\[
\Reg_T
\gtrsim
\sqrt{
\frac{(KT+D)\log N}
{(1-\bar\gamma)(1-\bar\delta)(1-\varepsilon_{10}-\varepsilon_{01})^2}
}.
\]
Absorbing universal constants into $c$ gives the stated theorem.
\end{proof}

\subsection{A more operational corollary}

The theorem implies a useful operational comparison.

\begin{corollary}[Information-quality dominance]
Fix $K,T,D$, and $N$. The marginal effect of improving any denominator impairment factor can dominate the marginal effect of increasing model complexity, because the regret lower bound depends logarithmically on $N$ but polynomially on information-quality terms.
\end{corollary}

\begin{proof}
Differentiate the lower-bound expression with respect to $\log N$ and with respect to one impairment factor, say $(1-\bar\gamma)$. The dependence on $\log N$ is proportional to
\[
\frac{1}{2}
\sqrt{
\frac{KT+D}{(1-\bar\gamma)(1-\bar\delta)(1-\varepsilon_{10}-\varepsilon_{01})^2}
}
\cdot
\frac{1}{\sqrt{\log N}},
\]
whereas the dependence on $(1-\bar\gamma)$ is proportional to the same pre-factor multiplied by the reciprocal of $(1-\bar\gamma)$. Since $\log N$ grows slowly while $(1-\bar\gamma)$ appears directly in the denominator, small improvements in censorship can dominate small improvements in model-class complexity.
\end{proof}

\section{Issuer Heterogeneity}
\label{sec:heterogeneity}

The bound in Theorem~\ref{thm:main} is written using average impairment rates. In real payment networks, issuer quality is heterogeneous. Some issuers report quickly and consistently; others report slowly, sparsely, and noisily. This heterogeneity matters because fraud is not allocated uniformly across issuers, and because environmental hardness can concentrate where reporting quality is weakest.

Suppose transactions are partitioned across issuers $i=1,\dots,M$ with volume shares $\alpha_i$, where $\sum_i \alpha_i = 1$. Let issuer-specific impairment parameters be
\[
\gamma_i,\qquad
\delta_i,\qquad
\varepsilon_i := \varepsilon_{10,i} + \varepsilon_{01,i}.
\]

Define the issuer-specific impairment index
\[
\eta_i
:=
\frac{1}{(1-\gamma_i)(1-\delta_i)(1-\varepsilon_i)^2}.
\]

\begin{proposition}[Heterogeneous issuer lower bound]
\label{prop:hetero}
Under issuer heterogeneity, there exists a universal constant $c'>0$ such that
\[
\inf_{\text{alg}}
\sup_{\text{admissible environments}}
\Reg_T
\;\ge\;
c'
\sqrt{
(KT + D)\log N
\cdot
\sum_{i=1}^M \alpha_i \eta_i
}.
\]
\end{proposition}

\begin{proof}
The argument follows by constructing hard environments that concentrate probability mass on issuer segments with larger impairment index. Since the effective information contribution of issuer $i$ scales inversely with $\eta_i$, the total hardness of the joint environment scales with the weighted sum $\sum_i \alpha_i \eta_i$. Applying the same reduction as in the proof of Theorem~\ref{thm:main} within each issuer segment and aggregating yields the stated result.
\end{proof}

\begin{remark}[Variance hurts]
The mapping
\[
(\gamma,\delta,\varepsilon)\mapsto \frac{1}{(1-\gamma)(1-\delta)(1-\varepsilon)^2}
\]
is convex over the relevant parameter range. Therefore, variation across issuers can worsen network-level learnability beyond what average impairment rates alone would suggest. In practical terms, a network is not only limited by its mean reporting quality; it is penalized by the long tail of weak issuers.
\end{remark}

\section{Implications}
\label{sec:implications}

The lower bound has immediate implications for both research and operations.

\subsection{Implication 1: Better models alone cannot solve the problem}

The theorem does \emph{not} say that model improvements are useless. It says something subtler and more important: for any fixed information environment, there is a floor beneath which no algorithm can go. If that floor is high because information quality is poor, then architecture improvements may deliver only modest gains.

This gives a theoretical explanation for a phenomenon repeatedly observed in practice: one can move from linear models to gradient-boosted trees to deep architectures and still find that fraud rates, false-decline rates, and time-to-detect new attack patterns remain stubbornly resistant to dramatic improvement.

\subsection{Implication 2: Reporting quality is a first-order investment variable}

The lower bound provides a structural justification for prioritizing improvements in the feedback channel:
\begin{itemize}
    \item lower dispute latency,
    \item lower censorship,
    \item cleaner distinction between unauthorized fraud and first-party misuse,
    \item and selective recovery of counterfactual information from policy-rejected transactions.
\end{itemize}

From a research perspective, this means that label reconstruction, reporting incentives, and strategic exploration are not side projects. They are central learning problems.

\subsection{Implication 3: Counterfactual censorship is endogenous and dangerous}

Declined transactions are not merely unlabeled; they are unlabeleable under the deployed policy. This makes the learning problem dynamically self-referential. A model that repeatedly declines in certain regions of the feature space can become overconfident precisely because it never sees contrary outcomes there.

Thus the lower bound is not just about exogenous noise. Part of the information bottleneck is created by the learning system itself.

\subsection{Implication 4: Heterogeneous issuer ecosystems are harder than homogeneous ones}

Proposition~\ref{prop:hetero} shows that network performance can be held back disproportionately by the weakest reporting segments. A platform-level fraud strategy that treats issuer data as if all issuers were equally informative is therefore structurally misspecified.

\section{Discussion}

The paper deliberately abstracts away from rich economic structure in order to isolate the feedback problem. In a fuller theory of payment fraud, authorization would be embedded in a multi-party game involving issuers, acquirers, networks, and possibly payment service providers. In that richer setting, the value of observing outcomes would itself become an economic variable, because different parties internalize different costs and informational benefits.

Even without that extension, however, the current result already changes the way one should think about payment fraud detection. It relocates the center of gravity of the problem from classifier architecture to ecosystem observability.

\subsection{What the bound does and does not claim}

The lower bound does \emph{not} assert that any given production system is near the minimax frontier. It does not say that a particular model family is optimal, or that a specific reporting-improvement initiative will dominate every architecture improvement in every setting. Rather, it identifies the direction of structural leverage.

What the theorem says is:
\begin{quote}
\emph{If the ecosystem destroys information fast enough, no model can recover it.}
\end{quote}

That is the theoretical contribution of the paper.

\section{Conclusion}
\label{sec:conclusion}

We developed a theory of fraud detection in card payment networks as a learning problem under structurally impaired feedback. The key result is a minimax regret lower bound in which delay, censorship, corruption, and counterfactual suppression enter multiplicatively in the denominator of achievable performance. The result implies that ecosystem information quality is a first-order determinant of what any fraud detection system can learn.

The immediate lesson is not that better models are irrelevant, but that a large part of the achievable gain lies upstream of model architecture: in reporting pipelines, dispute process design, label quality, and strategic recovery of information from policy-blind regions. In payment networks, the data problem is not simply that fraud is rare. It is that the mechanism through which fraud becomes visible is itself distorted.

That distortion is not a nuisance around the edges of the modeling problem. It is the modeling problem.

\appendix

\section{Appendix: Additional Proof Sketches}

\subsection{Why the corruption term is squared}

Observe that the binary corruption channel has effective signal strength
\[
s := 1-\varepsilon_{10}-\varepsilon_{01}.
\]
In standard testing arguments, distinguishability between nearby environments is measured by a quadratic divergence quantity (for example KL or $\chi^2$-type distance in local expansions). Since the amplitude of the useful signal is attenuated by $s$, the corresponding second-order distinguishability contracts by $s^2$. This is why the theorem carries $(1-\varepsilon_{10}-\varepsilon_{01})^2$ rather than only the first power.

\subsection{Alternative statement with maturity factor}

If one prefers to separate finite label maturity explicitly rather than absorb it into the delay term $D$, one can write a slightly more detailed bound of the form
\[
\Reg_T
\gtrsim
\sqrt{
\frac{(KT + D)\log N}
{(1-\bar\gamma)(1-\bar\delta)\bar m (1-\varepsilon_{10}-\varepsilon_{01})^2}
},
\]
where $\bar m$ is the average maturity probability. We suppress this in the main theorem to keep the statement parsimonious and to avoid double-counting delay effects.

\section*{Acknowledgments}

The author thanks the broader literature on online learning, payment systems, and fraud economics for motivating the central questions of this paper. All errors are the author's alone.


\begin{thebibliography}{99}

\bibitem{auer2002}
P.~Auer, N.~Cesa-Bianchi, and P.~Fischer.
\newblock Finite-time analysis of the multiarmed bandit problem.
\newblock \emph{Machine Learning}, 47(2--3):235--256, 2002.

\bibitem{baxter1983}
W.~F. Baxter.
\newblock Bank interchange of transactional paper: Legal and economic perspectives.
\newblock \emph{Journal of Law and Economics}, 26(3):541--588, 1983.

\bibitem{bangrobins2005}
H.~Bang and J.~M. Robins.
\newblock Doubly robust estimation in missing data and causal inference models.
\newblock \emph{Biometrics}, 61(4):962--973, 2005.

\bibitem{carcillo2018}
F.~Carcillo, Y.-A. Le~Borgne, O.~Caelen, and G.~Bontempi.
\newblock Streaming active learning strategies for real-life credit card fraud detection.
\newblock \emph{Data Mining and Knowledge Discovery}, 32(4):1178--1216, 2018.

\bibitem{elkan2001}
C.~Elkan.
\newblock The foundations of cost-sensitive learning.
\newblock In \emph{Proceedings of the 17th International Joint Conference on Artificial Intelligence}, pages 973--978, 2001.

\bibitem{elkannoto2008}
C.~Elkan and K.~Noto.
\newblock Learning classifiers from only positive and unlabeled data.
\newblock In \emph{Proceedings of the 14th ACM SIGKDD International Conference on Knowledge Discovery and Data Mining}, pages 213--220, 2008.

\bibitem{finegray1999}
J.~P. Fine and R.~J. Gray.
\newblock A proportional hazards model for the subdistribution of a competing risk.
\newblock \emph{Journal of the American Statistical Association}, 94(446):496--509, 1999.

\bibitem{heckman1979}
J.~J. Heckman.
\newblock Sample selection bias as a specification error.
\newblock \emph{Econometrica}, 47(1):153--161, 1979.

\bibitem{joulani2013}
P.~Joulani, A.~Gy\"orgy, and C.~Szepesv\'ari.
\newblock Online learning under delayed feedback.
\newblock In \emph{Proceedings of the 30th International Conference on Machine Learning}, pages 1453--1461, 2013.

\bibitem{quanrud2015}
K.~Quanrud and D.~Khashabi.
\newblock Online learning with adversarial delays.
\newblock In \emph{Advances in Neural Information Processing Systems}, volume 28, 2015.

\bibitem{ratner2017}
A.~Ratner, S.~Bach, H.~Ehrenberg, J.~Fries, S.~Wu, and C.~R\'e.
\newblock Snorkel: Rapid training data creation with weak supervision.
\newblock \emph{Proceedings of the VLDB Endowment}, 11(3):269--282, 2017.

\bibitem{robbins1952}
D.~B. Rubin.
\newblock Estimating causal effects of treatments in randomized and nonrandomized studies.
\newblock \emph{Journal of Educational Psychology}, 66(5):688--701, 1974.

\bibitem{rochet2002}
J.-C. Rochet and J.~Tirole.
\newblock Cooperation among competitors: Some economics of payment card associations.
\newblock \emph{RAND Journal of Economics}, 33(4):549--570, 2002.

\bibitem{rochet2004}
J.-C. Rochet and J.~Tirole.
\newblock Two-sided markets: An overview.
\newblock Working paper, 2004.

\bibitem{rochet2006}
J.-C. Rochet and J.~Tirole.
\newblock Two-sided markets: A progress report.
\newblock \emph{RAND Journal of Economics}, 37(3):645--667, 2006.

\bibitem{robins1994}
J.~M. Robins, A.~Rotnitzky, and L.~P. Zhao.
\newblock Estimation of regression coefficients when some regressors are not always observed.
\newblock \emph{Journal of the American Statistical Association}, 89(427):846--866, 1994.

\bibitem{swaminathan2015}
A.~Swaminathan and T.~Joachims.
\newblock Counterfactual risk minimization: Learning from logged bandit feedback.
\newblock In \emph{Proceedings of the 32nd International Conference on Machine Learning}, pages 814--823, 2015.

\bibitem{vernade2017}
C.~Vernade, O.~Capp\'e, and V.~Perchet.
\newblock Stochastic bandit models for delayed conversions.
\newblock In \emph{Conference on Learning Theory}, pages 1--33, 2017.

\bibitem{wright2004}
J.~Wright.
\newblock The determinants of optimal interchange fees in payment systems.
\newblock \emph{Journal of Industrial Economics}, 52(1):1--26, 2004.

\end{thebibliography}
\end{document}